\documentclass[11pt]{article}
\usepackage{amsmath,epsfig, amsthm, amssymb, enumerate}
\usepackage{fancyhdr}
\usepackage{fullpage}
\usepackage{graphicx}
\usepackage{subfigure}
\usepackage{amsmath, amssymb, epsfig}
\usepackage{bm}
\usepackage{mathrsfs}
\usepackage{color}

\DeclareMathOperator*{\argmin}{arg min} 
 
\DeclareMathOperator*{\supp}{supp}

\DeclareMathOperator*{\trace}{trace}

\newcommand{\defby}{\overset{\mathrm{\scriptscriptstyle{def}}}{=}}

\newtheorem{bigthm}{Theorem}

\newcommand{\R}{\mathbb{R}}
\newcommand{\C}{\mathbb{C}}
\newcommand{\Z}{\mathbb{Z}}
\newcommand{\N}{\mathbb{N}}

\newcommand{\vct}[1]{\bm{#1}}
\newcommand{\mtx}[1]{\bm{#1}}

 \newcommand{\scalprod}[2]{\left\langle #1,#2 \right\rangle}

\newtheorem{theorem}{Theorem}[]
\newtheorem{lemma}[theorem]{Lemma}
\newtheorem{proposition}[theorem]{Proposition}

\newtheorem{definition}[theorem]{Definition}

\begin{document}

\title{Stable image reconstruction using total variation minimization}

\author{Deanna Needell \and Rachel Ward\thanks{Supported in part by a Donald D. Harrington Faculty Fellowship, Alfred P. Sloan Research Fellowship, and DOD-Navy grant N00014-12-1-0743}}

\maketitle

\begin{abstract}
This article presents near-optimal guarantees for stable and robust image recovery from undersampled noisy measurements using total variation minimization.  In particular, we show
that from $O(s\log(N))$ nonadaptive linear measurements, an image can be
reconstructed to within the best $s$-term approximation of
its gradient up to a logarithmic factor, and this factor can be removed by taking slightly more measurements.  Along the way, we prove a strengthened Sobolev inequality for functions lying in the null space of suitably incoherent matrices.
\end{abstract}

\section{Introduction}

Compressed sensing (CS) provides the technology to exploit sparsity when acquiring signals of general interest, allowing for accurate and robust signal acquisition from surprisingly few measurements.  Rather than acquiring an entire signal and then later compressing, CS proposes a mechanism to collect measurements in compressed form, skipping the often costly step of complete acquisition.  The applications are numerous, and range from image and signal processing to remote sensing and error correction~\cite{CSwebpage}.  

In compressed sensing one acquires a signal $\vct{x} \in \C^d$ via $m \ll d$ linear measurements of the form ${y_k} = \langle \vct{\phi_k}, \vct{x} \rangle + {z_k}$.  The vectors $\vct{\phi_k}$ form the rows of the \textit{measurement matrix} $\mtx{\Phi}$, and the measurement vector $\vct{y} \in \C^m$ can thus be viewed in matrix notation as
$$
\vct{y} = \mtx{\Phi}\vct{x} + \vct{z},
$$
where $\vct{z}$ is the noise vector modeling measurement error.  We then ask to recover the signal of interest $\vct{x}$ from the noisy measurements $\vct{y}$.  Since $m\ll d$ this problem is ill-posed without further assumptions.  However, signals of interest in applications contain far less information than their dimension $d$ would suggest, often in the form of sparsity or compressibility in a given basis.   We call a vector $\vct{x}$  $s$-sparse when
\begin{equation}\label{eq:sparse}
\|\vct{x}\|_0 \defby |\supp(\vct{x})| \leq s \ll d.
\end{equation}
Compressible vectors are those which are approximated well by sparse vectors. 

In the simplest case, if we know that $\vct{x}$ is $s$-sparse and the measurements are free of noise, then the inverse problem $\vct{y} = \mtx{\Phi}\vct{x}$ is well-posed if the measurement matrix $\mtx{\Phi}$ is one-to-one on sparse vectors.  To recover $\vct{x} \in \C^d$ from $\vct{y} \in \C^m$ we solve the optimization problem
\begin{equation}\tag{$L_0$}
\hat{\vct{x}} = \argmin_{\vct{w}} \|\vct{w}\|_0 \quad\text{such that}\quad \mtx{\Phi}\vct{w} = \vct{y}.
\end{equation}

If $\mtx{\Phi}$ is one-to-one on $s$-sparse vectors and $\vct{x}$ is $s$-sparse, then $(L_0)$ recovers $\vct{x}$ exactly:  $\hat{\vct{x}} = \vct{x}$.  The optimization problem $(L_0)$ however is in general NP-Hard~\cite{N95} so we instead consider its relaxation to the $\ell_1$-norm,
\begin{equation}\tag{$L_1$}
\hat{\vct{x}} = \argmin_{\vct{w}} \|\vct{w}\|_1 \quad\text{such that}\quad \|\mtx{\Phi}\vct{w} - \vct{y}\|_2 \leq \varepsilon,
\end{equation}
where $\|\vct{w}\|_1 = \sum_i |w_i|$ and $\|\vct{w}\|_2 = \left(\sum_i w_i^2\right)^{1/2}$, and $\varepsilon$ bounds the noise level $\|\vct{z}\|_2 \leq \varepsilon$.  The problem $(L_1)$ may be cast as a second order cone program (SOCP) and can thus be solved efficiently using modern convex programming methods~\cite{CDS99:Atomic-Decomposition,DT97:Linear}.

If we require that the measurement matrix is not only one-to-one on $s$-sparse vectors, but moreover an approximate \emph{isometry} on $s$-sparse vectors, then $(L_1)$ will not only recover $s$-sparse signals exactly, but also recover nearly sparse signals approximately.  Cand\`es et.al. introduced the \textit{restricted isometry property} (RIP) in~\cite{CT05:Decoding} as a sufficient condition on the measurement matrix $\mtx{\Phi}$ for guaranteed robust recovery of compressible signals via $(L_1)$. 

\begin{definition}\label{def:rip}
A matrix $\mtx{\Phi} \in \mathbb{C}^{m \times d}$ is said to have the restricted isometry property of order $s$ and level $\delta \in (0,1)$  if
\begin{align}
\label{rip}
(1 - \delta) \| \vct{x} \|_2^2 \leq \| \mtx{\Phi} \vct{x} \|_2^2 \leq (1 + \delta) \| \vct{x} \|_2^2 \hspace{12mm} \textrm{ for all}\hspace{2mm} s \textrm{-sparse } \vct{x} \in \mathbb{C}^d .
\end{align}
The smallest such $\delta$ for which this holds is denoted by $\delta_s$ and called the \emph{restricted isometry constant} for the matrix $\mtx{\Phi}$.
\end{definition}
When $\delta_{2s} < 1$, the RIP guarantees that no $2s$-sparse vectors reside in the null space of $\mtx{\Phi}$.  When a matrix has a small restricted isometry constant, $\mtx{\Phi}$ acts as a near-isometry over the subset of $s$-sparse signals.  

Many classes of random matrices can be used to generate matrices having small RIP constants.  With probability exceeding $1-e^{-Cm}$, a matrix  whose entries are i.i.d. appropriately normalized Gaussian random variables has a small RIP constant $\delta_s < c$ when $m \gtrsim c^{-2} s\log(d/s)$.   This number of measurements is also shown to be necessary for the RIP  \cite{krahmer2010new}.
More generally, the restricted isometry property holds with high probability for any matrix generated by a subgaussian random variable \cite{CT04:Near-Optimal,MPJ06:Uniform,RV08:sparse,badadewa08}.    One can also construct matrices with the restricted isometry property using fewer random bits.  For example, if $m \gtrsim s \log^4(d)$ then the restricted isometry property holds with high probability for the \emph{random subsampled Fourier matrix} $\mtx{F}_{\Omega} \in \C^{m \times d}$, formed by restricting the $d \times d$ discrete Fourier matrix to a random subset of $m$ rows and re-normalizing \cite{RV08:sparse}.  The RIP also holds for randomly subsampled bounded orthonormal systems \cite{rw11, ra09-1} and randomly-generated circulant matrices \cite{RRT12:circulant}.

Cand\`es, Romberg, and Tao showed that when the measurement matrix $\mtx{\Phi}$ satisfies the RIP with sufficiently small restricted isometry constant, $(L_1)$ produces an estimation $\hat{\vct{x}}$ to $\vct{x}$ with error~\cite{CRT06:Stable},
\begin{equation}\label{eq:L1}
\|\hat{\vct{x}} - \vct{x}\|_2 \leq C\left(\frac{\|\vct{x} - \vct{x_s}\|_1}{\sqrt{s}} + \varepsilon \right).
\end{equation}
This error rate is optimal on account of classical results about the Gel'fand widths of the $\ell_1$ ball due to Kashin~\cite{Kas77:The-widths} and Garnaev--Gluskin~\cite{GG84:On-widths}.

Here and throughout, $\vct{x_s}$ denotes the vector consisting of the largest $s$ coefficients of $\vct{x}$ in magnitude.  Similarly, for a set $S$, $\vct{x}_S$ denotes the vector (or matrix, appropriately)  of $\vct{x}$ restricted to the entries indexed by $S$.  The bound~\eqref{eq:L1} then says that the recovery error is proportional to the noise level and the norm of the tail of the signal, $\vct{x} - \vct{x_s}$.  As a special case, when the signal is exactly sparse and there is no noise in the measurements, $(L_1)$ recovers $\vct{x}$ exactly.  We note that for simplicity, we restrict focus to CS decoding via the program ($L_1$), but acknowledge that other approaches in compressed sensing such as Compressive Sampling Matching Pursuit~\cite{NT08:Cosamp} and Iterative Hard Thresholding~\cite{BD08:Iterative} yield analogous recovery guarantees.

Signals of interest are often compressible with respect to bases other than the canonical basis.  We consider a vector $\vct{x}$ to be $s$-sparse with respect to the basis $\mtx{B}$ if
$$
\vct{x} = \mtx{B}\vct{z} \quad\text{for some $s$-sparse $\vct{z}$,}\quad
$$
and $\vct{x}$ is compressible with respect to this basis when it is well approximated by a sparse representation.  In this case one may recover $\vct{x}$ from underdetermined linear measurements $\vct{y} = \mtx{\Phi}\vct{x} + \vct{\xi}$ using the modified $\ell_1$ minimization problem

\begin{equation}\tag{$BL_1$}
\hat{\vct{x}} = \argmin_{\vct{w}} \|\mtx{B}^{*}\vct{w}\|_1 \quad\text{such that}\quad \|\mtx{\Phi}\vct{w} - \vct{y}\|_2 \leq \varepsilon, 
\end{equation}
where here and throughout $\mtx{B}^*$ denotes the conjugate transpose or adjoint of the matrix $\mtx{B}$. 
As before, the recovery error $\| \vct{x} -\vct{\hat{x}}  \|_2$ is proportional to the noise level and the norm of the tail of the signal if the composite matrix $\mtx{\Psi} = \mtx{\Phi} \mtx{B}$ satisfies the RIP.  If $\mtx{B}$ is a fixed orthonormal matrix and $\mtx{\Phi}$ is a random matrix generated by a subgaussian random variable, then $\mtx{\Psi} = \mtx{\Phi} \mtx{B}$ has RIP with high probability with $m \gtrsim s \log(d/s)$ due to the invariance of  norm-preservation for subgaussian matrices \cite{badadewa08}.  More generally, following the approach of \cite{badadewa08} and applying Proposition 3.2 in \cite{krahmer2010new}, this rotation-invariance holds for any $\mtx{\Phi}$ with the restricted isometry property and randomized column signs.   The rotational-invariant RIP also extends to the classic $\ell_1$-analysis problem which solves $(BL_1)$ when $\mtx{B^*}$ is a tight frame~\cite{CENR_Compressed}.

\subsection{Imaging with compressed sensing}  

Grayscale digital images have lower-dimensional structure than their ambient number of pixels suggests, consisting primarily of slowly-varying pixel intensities except around edges in the underlying image.  In other words, digital images are compressible with respect to their discrete gradient.  Concretely, we denote an $N\times N$ block of pixels by $\mtx{X} \in \C^{N \times N}$, and we write $X_{j,k}$ to denote any particular pixel.  The discrete directional derivatives of $\mtx{X} \in \C^{N \times N}$ are defined pixel-wise as

\begin{eqnarray}
\mtx{X}_x : \C^{N \times N} \rightarrow \C^{(N-1) \times N}, \quad \quad (\mtx{X}_x)_{j,k} &=&  \mtx{X}_{j+1,k} - \mtx{X}_{j,k}  \label{Xx} \\
\mtx{X}_y : \C^{N \times N} \rightarrow \C^{N \times (N-1)}, \quad \quad
(\mtx{X}_y)_{j,k} &=& \mtx{X}_{j,k+1} - \mtx{X}_{j,k}
  \end{eqnarray}

The discrete gradient transform $\mtx{\nabla:} \C^{N \times N} \rightarrow \C^{N \times N \times 2}$ is defined
in terms of these directional derivatives and in matrix form, 
\begin{equation}
\label{grad}
\big[\mtx{\nabla} \mtx{X}\big]_{j,k} \defby \left\{ \begin{array}{ll}
\big( (\mtx{X}_x)_{j,k}, (\mtx{X}_y)_{j,k} \big), & 1 \leq j \leq N-1, \quad 1 \leq k \leq N-1 \nonumber \\
\big( 0, (\mtx{X}_y)_{j,k} \big), & j=N, \quad 1 \leq k \leq N-1 \nonumber \\
\big((\mtx{X}_x)_{j,k}, 0 \big), & k=N, \quad 1\leq j \leq N-1 \nonumber \\
\big(0,0 \big), & j=k=N
\end{array}
\right.
\end{equation}
Finally, the \emph{total variation} seminorm of ${\bf X}$ is the $\ell_1$ norm of its discrete gradient,
 
\begin{equation}\label{eq:TV}
\| \mtx{X} \|_{TV} \defby \|\mtx{\nabla} \mtx{X}\|_1.
\end{equation}

We note here that we have defined the \textit{anisotropic} version of the total variation seminorm.  The \textit{isotropic} version of the total variation seminorm corresponds to taking the $\ell_1$ norm of the vector with components $(\mtx{X}_x)_{j,k} + i(\mtx{X}_y)_{j,k}$, and becomes the sum of terms
$$
\big|(\mtx{X}_x)_{j,k} + i(\mtx{X}_y)_{j,k}\big| = \left( (\mtx{X}_x)_{j,k}^2 + (\mtx{X}_y)_{j,k}^2\right)^{1/2}.
$$
The isotropic and anisotropic induced total variation seminorms are thus equivalent up to a factor of $\sqrt{2}$.  While we will write all results in terms of the anisotropic total variation seminorm, our results also extend to the isotropic version, see \cite{NW12:TV3} for more details.  

As natural images are well-approximated as piecewise-constant, it makes sense to choose from among the infinitely-many images agreeing with a set of underdetermined linear measurements the one having smallest total variation.   In the context of compressed sensing, the measurements $\vct{y}\in\C^m$ from an image $\mtx{X}$ are of the form $\vct{y} = {\cal M}(\mtx{X}) + \vct{\xi}$, where $\vct{\xi}$ is a noise term with bounded norm $\| \vct{\xi} \|_2 \leq \epsilon$, and ${\cal M} : \C^{N\times N}\rightarrow \C^m$ is a linear operator defined via its components by
$$
[{\cal M}(\mtx{X})]_j \defby \langle \mtx{M_j}, \mtx{X} \rangle = \trace(\mtx{M_j}\mtx{X}^*),
$$
for suitable matrices $\mtx{M_j}$.  Total variation minimization refers to the convex optimization problem
\begin{equation}\tag{TV}
\mtx{\hat{X}} = \argmin_{\mtx{Z}} \| \mtx{Z} \|_{TV}  \quad  \textrm{ such that }  \quad \| {\cal M}(\mtx{Z}) -\vct{y} \|_2   \leq \varepsilon.
\end{equation}
The standard theory of compressed sensing does not apply to total variation minimization. In fact, the gradient transform $\mtx{Z} \rightarrow \mtx{\nabla} \mtx{Z}$ not only fails to be orthonormal, but viewed as an invertible operator over mean-zero images, the Frobenius operator norm of its inverse grows linearly with the discretization level $N$.  Still, total variation minimization is widely used in compressed sensing applications and exhibits accurate image reconstruction empirically (see e.g.~\cite{CRT06:Stable,crt,candes05pr,osher2003image,chan2006total,lustig2007sparse,lustig2008compressed,liu2011total,nett2008tomosynthesis,mayin,kai2008suppression,keeling2003total}).  However, to the authors' best knowledge there have been no provable guarantees that $(TV)$ recovery is robust.

Images are also compressible with respect to wavelet transforms.  The Haar transform (and wavelet transforms more generally) is multi-scale, collecting information not only about local differences in pixel intensity, but also differences in average pixel intensities across all dyadic scales.  It should not be surprising then that the level of compressibility in the wavelet domain can be  controlled by the total variation seminorm ~\cite{devore1992image}.   In particular, we will use a result of Cohen, DeVore, Petrushev, and Xu which says that the rate of decay of the bivariate Haar wavelet coefficients of an image can be bounded by the total variation (see Proposition~\ref{cor:cdpx} in Section~\ref{sec:Poincare}).   

Recall that the (univariate) Haar wavelet system constitutes a complete orthonormal system for square-integrable functions on the unit interval, consisting of the constant function
$$
H^{0}(t) = \left\{ \begin{array}{ll} 1 & 0 \leq t < 1, \nonumber \\
0, & \textrm{otherwise},
\end{array} \right.
$$
the mother wavelet
\begin{equation}
\label{haarmother}
H^1(t) = \left\{ \begin{array}{ll} 1 & 0 \leq t < 1/2, \nonumber \\
-1 & 1/2 \leq t < 1,
\end{array} \right.
\end{equation}
and dyadic dilations and translates of the mother wavelet
\begin{equation}
\label{haarsystem}
H_{n,k}(t) = 2^{n/2} H^{1}(2^n t - k); \quad n \in \N, \quad 0 \leq k < 2^n.
\end{equation}
The bivariate Haar wavelet basis is an orthonormal basis for $L_2(Q)$, the space of square-integrable functions on the unit square $Q = [0,1)^2$, and is derived from the univariate Haar system by the usual tensor-product construction.  In particular, starting from the multivariate functions 
$$
H^{e}(u,v) = H^{e_1}(u) H^{e_2}(v),  \quad e = (e_1, e_2) \in V = \big\{ \{0,1\}, \{1,0\}, \{1,1\} \big\},
$$
the bivariate Haar system consists of the constant function $H^0(u,v) \equiv 1$, and all functions 
\begin{equation}
\label{bivariateH1}
x=(u,v), \quad \quad H_{j,k}^e(x) = 2^j H^e(2^j x - k), \quad e \in V, \quad j \geq 0, \quad k \in \Z^2 \cap 2^j Q 
\end{equation}
Discrete images are isometric to the space $\Sigma_N \subset L_2(Q)$ of piecewise-constant functions
\begin{equation}
\label{isometry}
\Sigma_{N} = \left\{ f \in L_2(Q), \quad f(u,v) = c_{j,k}, \quad \frac{j-1}{N} \leq u < \frac{j}{N}, \quad  \frac{k-1}{N} \leq v < \frac{k}{N}\right\}
\end{equation}
via the identification $c_{j,k} = N X_{j,k}$. 
Let $N = 2^n$, and consider the bivariate Haar basis restricted to the $N^2$ basis functions 
$H^0 \cup \{ H_{j,k}^e \}_{0 \leq j \leq n-1, \hspace{.5mm} k \in \mathbb{Z}^2 \cap 2^j Q}^{ e \in V}$.  Identified via \eqref{isometry} as discrete images $\vct{h}^0$ and 
$(\vct{h}_{j, k}^e)$ respectively, this system forms an orthonormal basis for $\C^{N \times N}$.  We denote by ${\cal H}(\mtx{X})$ the matrix product that computes the \emph{discrete bivariate Haar transform} $\mtx{X} \rightarrow (\scalprod{\mtx{X}}{\vct{h}^0}, \scalprod{ \mtx{X}}{ \vct{h}_{j,k}^e} )$.  

Because the bivariate Haar transform is orthonormal,  standard CS results guarantee that images can be reconstructed up to a factor of their best approximation by $s$ Haar basis functions using $m \gtrsim s\log(N)$ measurements.   One might then consider $\ell_1$-minimization of the Haar coefficients, that is, $(BL_1)$ with orthonormal transform $\mtx{B} = {\cal H}$, as an alternative to total variation minimization.    However, total variation minimization gives better empirical image reconstruction results than $\ell_1$-Haar wavelet coefficient minimization, despite not being fully justified by compressed sensing theory. For details, see~\cite{candes05pr,CRT06:Stable,duarte08si} and references therein.  
For example, Figure~\ref{fig:camera} compares reconstructions of the Cameraman image from only $20\%$ of its discrete Fourier coefficients, using total variation minimization and $\ell_1$-Haar minimization.   

When the measurements are corrupted by additive noise, $\vct{y} = \mtx{\Phi}\vct{x} + \vct{\xi}$, the story is similar.  Figure~\ref{fig:LenaNoise} displays the original Fabio image, corrupted with additive Gaussian noise.  Again, we compare the performance of (TV) and $(BL_1)$ at reconstruction using $20\%$ Fourier measurements.   As is evident, TV-minimization outperforms Haar minimization in the presence of noise as well.  Another type of measurement noise is a consequence of round-off or quantization error.  This type of error may stem from the inability to take measurements with arbitrary precision, and differs from Gaussian noise since it depends on the signal itself.  Figure~\ref{fig:FabioNoise} displays the lake image with quantization error along with the recovered images.  As in the case of Gaussian noise, TV-minimization outperforms Haar minimization.
All experiments here and throughout used the software $\ell_1$-magic to solve the minimization programs~\cite{L1Magic}.

\begin{figure}[ht]
\centering

\subfigure[]{
   \includegraphics[width=1in] {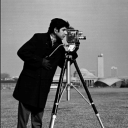}
 }
 \subfigure[]{
   \includegraphics[width=1in] {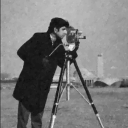}
 }
 \subfigure[]{
   \includegraphics[width=1in, height=1in] {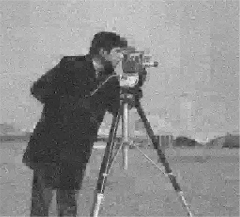}
 }

\caption{(a) Original $256 \times 256$ Cameraman image and its reconstruction from $20\%$ of its Fourier coefficients using (b) total variation minimization  and (c) $\ell_1$ minimization of its bivariate Haar coefficients.}\label{fig:camera}
\end{figure}

\begin{figure}[ht]
\centering

\subfigure[]{
   \includegraphics[width=1in] {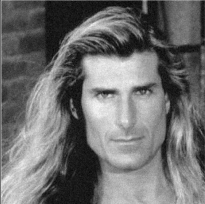}
 }
 \subfigure[]{
   \includegraphics[width=1in] {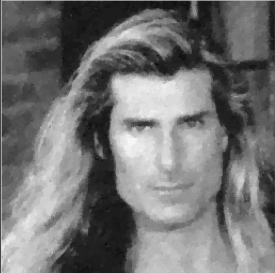}
 }
 \subfigure[]{
   \includegraphics[width=1in] {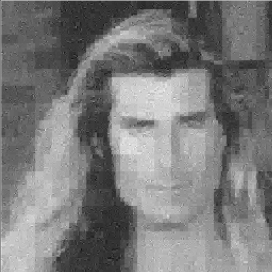}
 }

\caption{(a) Original $256 \times 256$ Fabio image corrupted with Gaussian noise and its reconstruction from $20\%$ of its Fourier coefficients using (b) total variation minimization and (c) $\ell_1$-minimization of its bivariate Haar coefficients}\label{fig:LenaNoise}
\end{figure}

\begin{figure}[ht]
\centering

\subfigure[]{
   \includegraphics[width=1in] {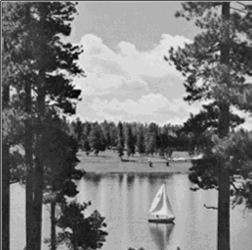}
 }
 \subfigure[]{
   \includegraphics[width=1in] {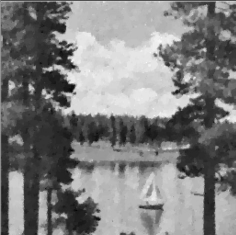}
 }
 \subfigure[]{
   \includegraphics[width=1in] {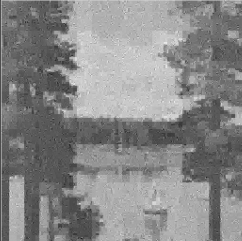}
 }

\caption{(a) Original $256 \times 256$ lake image corrupted with quantization noise and its reconstruction from $20\%$ of its  Fourier coefficients using (b) total variation minimization and (c) $\ell_1$-minimization of its bivariate Haar coefficients.}\label{fig:FabioNoise}
\end{figure}

We note that the use of total variation regularization in image processing predates the theory of compressed sensing.  The seminal paper of Rudin, Osher, and Fatemi introduced total variation regularization in imaging \cite{rudin1992nonlinear} and subsequently total variation has become a regularizer of choice for image denoising, deblurring, impainting, and segmentation  ~\cite{candes2002new,osher2003image,strong2003,chan2006total,chambolle2004}.  For more details on the connections between total variation minimization and wavelet frame-based methods in image analysis, we refer the reader to \cite{camreport}.

\subsection{Contribution of this paper}

We show that there are choices of underdetermined linear measurements (constructed from RIP matrices) for which the total variation minimization program $(TV)$ is guaranteed to recover images stably and robustly up to the best $s$-term approximation of their gradient.  The error guarantees are analogous to those of~\eqref{eq:L1} up to a logarithmic factor, which we show can be removed by taking slightly more measurements (see Theorem~\ref{thm2} below).  Precisely, we have 

\begin{bigthm}\label{introThm}
Fix integers $m, N,$ and $s$ such that $m \geq C_1s \log(N^2/s)$.  There exist linear operators ${\cal M}: \C^{N \times N} \rightarrow \C^{m}$ for which the following holds for all $\mtx{X} \in \C^{N \times N}.$  Suppose we observe noisy measurements $\vct{y} = {\cal M}(\mtx{X}) + \vct{\xi}$ with noise level $\| \vct{\xi} \|_2 \leq \varepsilon$. Then the solution 
\begin{eqnarray}
\mtx{\hat{X}} = \argmin_{\mtx{Z}} \| \mtx{Z} \|_{TV}  \quad  \textrm{such that}  \quad \| {\cal M}(\mtx{Z}) - \vct{y} \|_2   \leq \varepsilon
\end{eqnarray}
satisfies 
\begin{equation}
\| \mtx{X}-\mtx{\hat{X}} \|_2 \leq  C_2\log(N^2/s)\Big( \frac{ \| \mtx{\nabla}\mtx{X} - (\mtx{\nabla}\mtx{X})_s \|_1}{\sqrt{s}}+ \varepsilon\Big).
\end{equation}
Here, $C_1$ and $C_2$ are universal constants independent of everything else.
\end{bigthm}
 For details about the construction of the measurements, see Theorem~\ref{thm:gen} and the remarks following.

\subsection{Previous work on TV minimization in compressed sensing}

The last few years have witnessed numerous algorithmic advances that allow the efficient implementation of total variation minimization (TV), such as the split Bregman algorithm proposed by \cite{oshergold}, based on the Bregman distance \cite{bregman}.  Several algorithms are designed to exploit the structure of Fourier measurements for further speed-up; see for example \cite{yang10, dias07}.  Image reconstruction via independent minimization of the directional derivatives $\mtx{X}_x$ and $\mtx{X}_y$ was observed in \cite{cdpx} to give superior empirical results. 

With respect to theory,~\cite{crt} showed that if an image $\mtx{X}$ has an exactly sparse gradient, then $(TV)$ recovers the image exactly from a small number of partial Fourier measurements.   Moreover, because the discrete Fourier transform commutes with the discrete gradient operator, one may change coordinates in this case and re-cast $(TV)$ as an $\ell_1$ program $(L_1)$ with respect to the discrete gradient image  \cite{pmgc} to derive stable \emph{gradient} recovery results.   

However, robust recovery of the gradient need not imply robust recovery of the image itself.  To see this, suppose the error $\mtx{\nabla}\mtx{X} -  \mtx{\nabla}\hat{\mtx{X}}$ in the recovery of the gradient has a single non-zero component, of size $\alpha$, located at pixel $(1,1)$.  That is, the gradient is recovered perfectly except at one pixel location, namely the upper left corner.  Then based on this alone, it is possible that every pixel in $\mtx{\hat{X}}$ differs from that in $\mtx{X}$ by the amount $\alpha$.  This accumulation of error means that even when the reconstructed gradient is close to the gradient of $\mtx{X}$, the images $\mtx{\hat{X}}$ and $\mtx{X}$ may be drastically different, magnified by a factor of $N^2$.  Even for mean-zero images, the error may be magnified by a factor of $N$, as for images $\mtx{X}$ with pixels $X_{j,k} = j$.  We show that due to properties of the null space of RIP matrices,  the (TV) reconstruction error $\mtx{X} - \mtx{\hat{X}}$ in \eqref{introThm} cannot propagate as such.

Recent work in~\cite{ndeg:cosparse} presents an \textit{analysis co-sparse model} which considers signals sparse in the analysis domain.  A series of theoretical and numerical tools are developed to solve the analysis problem $(BL_1)$ in a general framework.  In particular, the analysis operator may be the finite difference operator, which concatenates the vertical and horizontal derivatives into a single vector and is thus closely linked with the total variation operator.  Effective pursuit methods are also proposed to solve such problems under the analysis co-sparse prior assumption.  We refer the reader to~\cite{ndeg:cosparse} for details.

We note that our robustness recovery results for (TV) are specific to two-dimensional images,  as the embedding theorems we rely on do not hold for one-dimensional arrays.  Thus, our results do not imply robust recovery for one-dimensional piecewise constant signals.  Robustness for the recovery of the gradient support for piecewise constant signals was studied in~\cite{sparseanalysis2012}.  On the other hand, the results in this paper were recently extended in~\cite{NW12:TV3} to higher dimensional signals, $\mtx{X}\in\C^{N^d}$ for $d \geq 3$.

\subsection{Organization}  The paper is organized as follows.  Section~\ref{sec:main} contains the statement of our main results about robust total variation recovery.  The proof of our main results will occupy most of the remainder of the paper.  We first prove robust recovery of the image gradient in Section~\ref{sec:grad}.  In Section~\ref{sec:Poincare} we derive a strong Sobolev inequality for discrete images lying in the null space of an RIP matrix which will bound the image recovery error by its total variation.   Our result relies on a result by Cohen, DeVore, Petrushev, and Xu that the compressibility of the bivariate Haar wavelet transform is controlled by the total variation of an image.  We prove Theorem~\ref{introThm} by way of Theorem~\ref{thm:gen} in Section~\ref{sec:prfgen}.  We prove Theorem~\ref{thm2}, showing that the logarithmic factor of Theorem~\ref{introThm} can be removed by taking slightly more measurements in Section \ref{sec:thm2}.  We conclude in Section~\ref{sec:conc} with some brief discussion.  Proofs of intermediate propositions are included in the appendix. 

\section{Main results}\label{sec:main}

Our main results use the following proposition which generalizes the results used implicitly in the recovery of sparse signals using $\ell_1$ minimization.  It allows us to bound the norm of an entire signal when the signal (a) is close to the null space of an RIP matrix and (b) obeys an $\ell_1$ cone constraint.   In particular,  \eqref{eq:h2} is just a generalization of results in \cite{crt}, while \eqref{eq:h1} follows from \eqref{eq:h2} and the cone-constraint \eqref{cc}.  The proof of Proposition \ref{cone-tube} is contained in the appendix.

\begin{proposition}
\label{cone-tube}
Suppose that ${\cal A}$ satisfies the restricted isometry property of order $5k \gamma^2$, for some $\gamma \geq 1$, and level $\delta < 1/3$, and suppose that the image $\mtx{D}$ satisfies a tube constraint
$$
\| {\cal A}(\mtx{D}) \|_2 \lesssim \varepsilon.
$$
Suppose further that for a subset $S$ of cardinality $|S| \leq k$, $\mtx{D}$ satisfies the cone-constraint
\begin{equation}
\label{cc}
\| \mtx{D}_{S^c} \|_1 \leq \gamma \| \mtx{D}_S \|_1 + \sigma.
\end{equation}
Then
\begin{equation}\label{eq:h2}
\| \mtx{D} \|_2 \lesssim \frac{\sigma}{\gamma \sqrt{k}} + \varepsilon
\end{equation}
 and
\begin{equation}\label{eq:h1}
\| \mtx{D} \|_1 \lesssim \sigma + \gamma \sqrt{k} \varepsilon.
\end{equation}
\end{proposition}
Neither the RIP level of $5k \gamma^2$ nor the restricted isometry constant $\delta < 1/3$ are sharp; for instance, an RIP level of $2s$ and restricted isometry constant $\delta_{2s} \approx .4931$ are sufficient for Proposition \ref{cone-tube} with $\gamma = 1$ \cite{moli,Can08:Restricted-Isometry}. 

 For simplicity of presentation, we say that a linear operator ${\cal A}: \C^{N_1 \times N_2} \rightarrow \C^m$ has the restricted isometry property (RIP) of order $s$ and level $\delta \in (0,1)$  if
\begin{align}
\label{Operatorrip}
(1 - \delta) \| \mtx{X} \|_2^2 \leq \| {\cal{A}}( \mtx{X}) \|_2^2 \leq (1 + \delta) \| \mtx{X} \|_2^2 \hspace{6mm} \textrm{ for all}\hspace{2mm} s \textrm{-sparse } \mtx{X} \in \mathbb{C}^{N_1\times N_2} .
\end{align}
Here and throughout, $\|\mtx{X}\|_p = \left(\sum_{j,k}|\mtx{X}_{j,k}|^p\right)^{1/p}$ denotes the entrywise $\ell_p$-norm of the image $\mtx{X}$, treating the image as a vector.  In particular, $p=2$ is the Frobenius norm 
$$
\|\mtx{X}\|_2 = \sqrt{\sum_{j,k}  \left| X_{j,k} \right|^2 } = \sqrt{\textrm{tr}(\mtx{X}\mtx{X}^* )}.
$$
This norm is generated by the image inner product
\begin{equation}
\label{imageinnerprod}
\langle \mtx{X}, \mtx{Y} \rangle = \trace(\mtx{X}\mtx{Y}^*).
\end{equation}
Note that if the linear operator ${\cal A}$ is given by
$$
\left({\cal A}(\mtx{X})\right)_j = \langle \mtx{A_j}, \mtx{X}\rangle,
$$
then ${\cal A}$ satisfies this RIP precisely when the matrix whose rows consist of $\mtx{A_j}$ unraveled into vectors satisfies the standard RIP as defined in~\eqref{def:rip}.  There is thus clearly a one-to-one correspondence between RIP for linear operators ${\cal A}: \C^{N_1 \times N_2} \rightarrow \C^m$ and RIP for matrices $\mtx{\Phi} \in \C^{m \times (N_1 N_2)}$, and we treat these notions as equivalent.  
Finally, we will use the notation $u \gtrsim v$ to indicate that there exists some absolute constant $C > 0$ such that $u \geq C v$.  We use the notation $u \lesssim v$ accordingly.  In this article, $C>0$ will always denote a universal constant that might be different in each occurrence.   

Before presenting the main results, it will be helpful to first determine what form an optimal error recovery bound takes in the setting of image reconstruction via total variation minimization.  In standard compressed sensing, the optimal minimax error rate from $m \gtrsim s\log(N^2/s)$ nonadaptive linear measurements is
\begin{equation}
\|\hat{\vct{x}} - \vct{x}\|_2 \lesssim \frac{\|\vct{x} - \vct{x_s}\|_1}{\sqrt{s}} + \varepsilon.
\end{equation}
In the setting of images, this implies that the best possible error rate from $m \gtrsim s \log(N^2/s)$ linear measurements is at best:
\begin{equation}
\label{matrix:optbound}
\|\hat{\mtx{X}} - \mtx{X}\|_2 \lesssim \frac{\| \mtx{\nabla}\mtx{X} - (\mtx{\nabla}\mtx{X})_s \|_1}{\sqrt{s}} + \varepsilon.
\end{equation}

Above,  $(\mtx{\nabla}\mtx{X})_s$ is the best $s$-sparse approximation to the discrete gradient $\mtx{\nabla}\mtx{X}$.  To see that we could not possibly hope for a better error rate, observe that if we could, we would reach a contradiction in light of the norm of the discrete gradient operator: $\| \mtx{\nabla}\mtx{Z} \|_2 \leq 4 \| \mtx{Z} \|_2$.

Theorem \ref{thm:gen} guarantees a recovery error proportional to \eqref{matrix:optbound} up to a single logarithmic factor $\log(N^2/s)$. That is, the recovery error of Theorem \ref{thm:gen} is optimal up to at most a logarithmic factor.  We see in Theorem~\ref{thm2} that by taking more measurements, we obtain the optimal recovery error, without the logarithmic term.
  
To change coordinates from pixel domain to gradient domain, it will be useful for us to consider matrices $\mtx{\Phi}_0$ and $\mtx{\Phi}^0$ obtained from a matrix $\mtx{\Phi}$ by concatenating a row of zeros to the bottom and top of $\mtx{\Phi}$, respectively.  Concretely, for a matrix $\mtx{\Phi} \in \C^{(N-1) \times N}$, we denote by ${\mtx{\Phi}^0} \in \C^{N \times N}$ the augmented matrix $\mtx{\Phi}^0$ with entries

\begin{equation}
(\mtx{\Phi^0})_{j,k} = \left\{ \begin{array}{ll} 0, & j = 1 \\
\Phi_{j-1,k}, & 2 \leq j \leq N
\end{array}  \right.
\label{zeropad}
\end{equation}
We denote similarly by $\mtx{\Phi_0}$ the matrix resulting by adding an additional row of zeros to the bottom of $\mtx{\Phi}$. 

We can relate measurements using the padded matrices \eqref{zeropad} of the entire image to measurements of its directional gradients, as defined in \eqref{Xx}.  The following relation can be verified by direct algebraic manipulation and so the proof is omitted.

\begin{lemma}
\label{padderiv}
Given $\mtx{X} \in \C^{N \times N}$ and $\mtx{\Phi} \in \C^{(N-1) \times N}$,
\begin{equation}
\label{padrelatex}
\scalprod{\mtx{\Phi}}{\mtx{X}_x} = \scalprod{\mtx{\Phi}^0}{\mtx{X}} - \scalprod{\mtx{\Phi}_0}{\mtx{X}} \nonumber
\end{equation}
and
\begin{equation}
\label{padrelatey}
\scalprod{\mtx{\Phi}}{\mtx{X}_y^T} = \scalprod{\mtx{\Phi}^0}{\mtx{X}^T} - \scalprod{\mtx{\Phi}_0}{\mtx{X}^T}, \nonumber
\end{equation}
where $\mtx{X}^T$ denotes the (non-conjugate) transpose of the matrix $\mtx{X}$.
\end{lemma} 
For a linear operator ${\cal A}: \C^{(N-1) \times N} \rightarrow \C^{m}$ with component measurements ${\cal A}(\mtx{X})_j = \scalprod{\mtx{A}_j}{\mtx{X}}$ we denote by ${\cal A}^0: \C^{N \times N} \rightarrow \C^{m}$ the linear operator with components $[{\cal A}^0(\mtx{X})]_j = \scalprod{(\mtx{A}^0)_j}{\mtx{X}}$.  We define ${\cal A}_0: \C^{N \times N} \rightarrow \C^m$ similarly.

We are now prepared to state the main results of this paper.

\begin{theorem}\label{thm:gen}
Consider $n, m_1, m_2,  s \in \N$, and let $N = 2^n$.  Let ${\cal A}: \C^{(N-1) \times N} \rightarrow \C^{m_1}$ and ${\cal A}': \C^{(N-1) \times N} \rightarrow \C^{m_1}$ 
 be such that the concatenated operator $\hat{\cal A}(\mtx{X}) = \big( {\cal A}(\mtx{X}), {\cal A}'(\mtx{X}) \big)$ has the restricted isometry property of order $5s$ and level $\delta < 1/3$.  Let ${\cal B}: \C^{N \times N} \rightarrow \C^{m_2}$ be such that, composed with the inverse bivariate Haar transform, ${\cal B}{\cal H}^{-1}: \C^{N \times N} \rightarrow \C^{m_2}$ has the restricted isometry property of order $2s$ and level $\delta < 1$.  

Let $m = 4m_1 + m_2$, and consider the linear operator ${\cal M}: \C^{N \times N} \rightarrow \C^{m}$ with components
\begin{equation}
\label{measure:M}
{\cal M}(\mtx{X}) = \Big( {\cal A}^0(\mtx{X}), {\cal A}_0(\mtx{X}), {\cal A'}^0(\mtx{X}^T), {\cal A'}_0(\mtx{X}^T),{\cal B}({\mtx{X}})\Big).
\end{equation}

If $\mtx{X} \in \C^{N \times N}$ with discrete gradient $\mtx{\nabla}\mtx{X}$ is acquired through noisy measurements $\vct{y} = {\cal M}(\mtx{X}) + \vct{\xi}$ with noise level $\| \vct{\xi} \|_2 \leq \varepsilon$, then 
\begin{eqnarray}
\label{tv_stablesignal}
\mtx{\hat{X}} = \argmin_{\mtx{Z}} \| \mtx{Z} \|_{TV}  \quad  \textrm{such that}  \quad \| {\cal M}(\mtx{Z}) - \vct{y} \|_2   \leq \varepsilon
\end{eqnarray}

satisfies 

\begin{equation}
\label{stable1}
\| \mtx{\nabla}\mtx{X}-\mtx{\nabla} \mtx{\hat{X}} \|_2 \lesssim  \frac{ \| \mtx{\nabla}\mtx{X}- (\mtx{\nabla}\mtx{X})_s \|_1}{\sqrt{s}} + \varepsilon,
\end{equation}

\begin{equation}
\label{stable2}
\| \mtx{X}-\mtx{\hat{X}} \|_{TV} \lesssim  \| \mtx{\nabla}\mtx{X} - (\mtx{\nabla} \mtx{X})_s \|_1 + \sqrt{s} \varepsilon,
\end{equation}
and
\begin{equation}
\label{stable3}
\| \mtx{X}-\mtx{\hat{X}} \|_2 \lesssim  \log(N^2/s)\Big( \frac{ \| \mtx{\nabla}\mtx{X} - (\mtx{\nabla}\mtx{X})_s \|_1}{\sqrt{s}}+ \varepsilon\Big).
\end{equation}
\end{theorem}
%
%
  Our second main result shows that, by allowing for more measurements, one obtains stable and robust recovery guarantees as in Theorem \ref{thm:gen} but with the log factor in \eqref{stable3} removed.  Moreover, the following theorem holds for general sensing matrices having restricted isometry properties.

\begin{theorem}\label{thm2}
Consider $n, m,  s \in \N$, and let $N = 2^n$.  There is an absolute constant $C > 0$ such that if ${\cal A}: \C^{N \times N} \rightarrow \C^{m}$ is such that, composed with the inverse bivariate Haar transform, ${\cal A}{\cal H}^{-1}: \C^{N \times N} \rightarrow \C^{m}$ has the restricted isometry property of  order $C s \log^3(N)$ and level $\delta < 1/3$, then the following holds for any $\mtx{X} \in \C^{N \times N}$.  If noisy measurements $\vct{y} = {\cal A}(\mtx{X}) + \vct{\xi}$ are observed with noise level $\| \vct{\xi} \|_2 \leq \varepsilon$, then 
$$
\mtx{\hat{X}} = \argmin_{\mtx{Z}} \| \mtx{Z} \|_{TV}  \quad  \textrm{such that}  \quad \| {\cal A}(\mtx{Z}) - \vct{y} \|_2   \leq \varepsilon
$$
satisfies 

\begin{equation}
\label{stable_easy}
\| \mtx{X}-\mtx{\hat{X}} \|_2 \lesssim  \frac{ \| \mtx{\nabla}\mtx{X} - (\mtx{\nabla}\mtx{X})_s \|_1}{\sqrt{s}}+ \varepsilon.
\end{equation}
\end{theorem}

{\bf Remarks.}

{\bf 1.}  In light of~\eqref{matrix:optbound}, the gradient error guarantees \eqref{stable1} and \eqref{stable2} provided by Theorem~\ref{thm:gen} are optimal, and the image error guarantee \eqref{stable3} is optimal up to a logarithmic factor, which we conjecture to be an artifact of the proof.  We also believe that the $4m_1$ measurements derived from ${\cal A}$ in Theorem \ref{thm1}, which are only used to prove stable gradient recovery, are not necessary and can be removed.  Theorem~\ref{thm2} provides optimal error recovery guarantees, at the expense of an additional factor of $C_0\log^3(N)$ measurements.

{\bf 2.}  The RIP requirements in Theorem~\ref{thm:gen} mean that the linear operators \\
${\cal A}^0, {\cal A}_0, {\cal A'}^0, {\cal A'}_0,$ and ${\cal B}$, can be generated using standard RIP matrix ensembles which are incoherent with the Haar wavelet basis.  For example, these measurements can be generated from a subgaussian random matrix $\mtx{\Phi} \in \R^{m \times N^2}$ with $m \gtrsim s\log(N^2/s)$.  Such constructions give rise to Theorem~\ref{introThm}.  Alternatively, these measurements could be generated from a partial Fourier matrix $\mtx{F}_{\Omega} \in \C^{m \times N^2}$ with $m \gtrsim s\log^5(N)$ and randomized column signs \cite{krahmer2010new}. 
We note that without randomized column signs, the partial Fourier matrix with uniformly subsampled rows is not incoherent with wavelet bases.  As shown in \cite{kw12}, the partial Fourier matrix with rows subsampled according to an appropriate power law density is incoherent with the Haar wavelet basis and can be applied in Theorems \ref{thm:gen} and \ref{thm2}. 

{\bf 3.} We have not tried to optimize the dependence of constants on the values of the restricted isometry parameters in the theorems.  Further refinements may yield improvements and tighter bounds throughout.

{\bf 4.} Theorems \ref{thm:gen} and \ref{thm2} require the image side-length to be a power of $2$, $N = 2^n$.
This is not actually a restriction, as an image of arbitrary side-length $N$ can be reflected horizontally and vertically to produce an at most $2N \times 2N$ image with the same total variation up to a factor of $4$.

The remainder of the article is dedicated to the proofs of Theorem \ref{thm:gen} and Theorem \ref{thm2}.
The proof of Theorem~\ref{thm:gen} is two-part:  we first prove the bounds \eqref{stable1} and \eqref{stable2} concerning stable recovery of the discrete gradient.  We then prove a strengthened  Sobolev inequality for images in the null space of an RIP matrix, and stable image recovery follows.  
The proof of Theorem \ref{thm2} is similar but more direct, and does not use a Sobolev inequality explicitly.

\section{Stable gradient recovery for discrete images}\label{sec:grad}
In this section we prove statements \eqref{stable1} and \eqref{stable2} from Theorem \ref{thm:gen}, showing that total variation minimization recovers the gradient image robustly.  

\subsection{Proof of stable gradient recovery, bounds \eqref{stable1} and \eqref{stable2}}\label{stabgrad}

Since the operator $\hat{{\cal A}}(\mtx{X}) = ( {\cal A}(\mtx{X}), {\cal A}'(\mtx{X}) )$ has the RIP, in light of Proposition~\ref{cone-tube}, if suffices to show that the discrete gradient $ \mtx{\nabla}( \mtx{X-\hat{X}})$, regarded as a vector in $\C^{N^2}$, satisfies the tube and cone constraints.

Let $\mtx{D} = \mtx{X} - \mtx{\hat{X}}$, and set $\mtx{L} = (\mtx{D}_x, \mtx{D}_y^T)$.  For convenience, let $P$ denote the map which takes the index of a non-zero entry in $\mtx{\nabla}\mtx{D}$ to its corresponding index in $\mtx{L}$.  Observe that by definition of the gradient, $\mtx{L}$ has the same norm as $\mtx{\nabla}\mtx{D}$.  That is, $\|\mtx{L}\|_2 = \|\mtx{\nabla}\mtx{D}\|_2$ and $\|\mtx{L}\|_1 = \|\mtx{\nabla}\mtx{D}\|_1$.  It thus now suffices to show that the matrix $\mtx{L}$ satisfies the tube and cone constraint.

Let $\mtx{A_1}, \mtx{A_2}, \ldots \mtx{A_{m_1}}, \mtx{A_1}', \mtx{A_2}', \ldots \mtx{A_{m_1}}'$ be such that
$$
{\cal A}(\mtx{Z})_j = \langle \mtx{A_j}, \mtx{Z}\rangle, \quad {\cal A'}(\mtx{Z})_j = \langle \mtx{A_j}', \mtx{Z}\rangle
$$ 

\begin{description}
\item[Cone Constraint.]  Let $S$ denote the support of the largest $s$ entries of $\mtx{\nabla}\mtx{X}$.  By minimality of $\mtx{\hat{X}} = \mtx{X} - \mtx{D}$ and feasibility of $\mtx{X}$,

\begin{align*}
\| (\mtx{\nabla} \mtx{X})_S \|_1 - \| (\mtx{\nabla}\mtx{D})_S\|_1 - \| (\mtx{\nabla}\mtx{X})_{S^c}\|_1 &+ \|(\mtx{\nabla}\mtx{D})_{S^c}\|_1 \\
&\leq \| (\mtx{\nabla}\mtx{X})_S - (\mtx{\nabla}\mtx{D})_S\|_1 + \| (\mtx{\nabla}\mtx{X})_{S^c} - (\mtx{\nabla}\mtx{D})_{S^c}\|_1\\
&= \|\mtx{\nabla}\mtx{\hat{X}}\|_1\\
&\leq \|\mtx{\nabla}\mtx{{X}}\|_1\\
&= \| (\mtx{\nabla}\mtx{{X}})_S\|_1 + \| (\mtx{\nabla}\mtx{{X}})_{S^c}\|_1
\end{align*}
Rearranging, this yields

$$
\| (\mtx{\nabla}\mtx{D})_{S^c} \|_1 \leq \| (\mtx{\nabla}\mtx{D})_S \|_1 + 2\| \mtx{\nabla}\mtx{X} - (\mtx{\nabla}\mtx{X})_s \|_1.
$$

Since $\mtx{L}$ has the same non-zero entries as $\mtx{\nabla}\mtx{D}$, this implies that $\mtx{L}$ satisfies the cone constraint

$$
\| \mtx{L}_{P(S)^c} \|_1 \leq \| (\mtx{\nabla}\mtx{D})_{P(S)} \|_1 + 2\| \mtx{\nabla}\mtx{X} - (\mtx{\nabla}\mtx{X})_s \|_1.
$$

By definition of $P$, note that $|P(S)| \leq |S| = s$.

\item[Tube constraint.] First note that $\mtx{D}$ satisfies a tube constraint,
\begin{eqnarray}
\| {\cal M}(\mtx{D}) \|_2^2 
&\leq& 2\| {\cal M}(\mtx{X}) - \vct{y} \|_2^2 + 2\| {\cal M}(\mtx{\hat{X}}) - \vct{y} \|_2^2 \nonumber \\
&\leq& 4 \varepsilon^2 \nonumber
\end{eqnarray}

Now by Lemma \ref{padderiv}, 
\begin{eqnarray}
|\scalprod{\mtx{A}_j}{\mtx{D}_x}|^2 &=& | \scalprod{[\mtx{A_j}]^0}{\mtx{D}} - \scalprod{[\mtx{A_j}]_0}{\mtx{D}} |^2 \nonumber\\
&\leq&  2| \scalprod{[\mtx{A_j}]^0}{\mtx{D}} |^2 + 2| \scalprod{[\mtx{A_j}]_0}{\mtx{D}} |^2
\end{eqnarray}
and
\begin{eqnarray}
|\scalprod{\mtx{A}_j'}{\mtx{D}_y^T}|^2 &=& \left| \scalprod{[\mtx{A_j'}]^0}{\mtx{D}^T} - \scalprod{[\mtx{A_j'}]_0}{\mtx{D}^T} \right|^2 \nonumber\\
&\leq&  2\left| \scalprod{[\mtx{A_j'}]^0}{\mtx{D}^T}\right|^2 + 2\left|\scalprod{[\mtx{A_j'}]_0}{\mtx{D}^T} \right|^2
\end{eqnarray}

Thus $\mtx{L}$ also satisfies a tube-constraint:
\begin{eqnarray}
\| [{\cal A}\; {\cal A}'](\mtx{L}) \|_2^2 &=& \sum_{j=1}^m |\scalprod{\mtx{A}_j}{\mtx{D}_x}|^2 +  |\scalprod{\mtx{A}_j'}{\mtx{D}_y^T}|^2 \nonumber \\
 &\leq& 2 \| {\cal M}(\mtx{D}) \|_2^2 
 \nonumber \\
 &\leq& 8\varepsilon^2.
\end{eqnarray}

\end{description}
Proposition~\ref{cone-tube} then completes the proof.

\section{A strengthened Sobolev inequality for incoherent null spaces}\label{sec:Poincare}

As a corollary of the classical Sobolev embedding of the space of functions of bounded variation $BV(\R^2)$ into $L_2(\R^2)$ \cite{bv2000}, the Frobenius norm of a zero-mean image is bounded by its total variation semi-norm. 

\begin{proposition}[Sobolev inequality for images]
\label{corav}
Let $\mtx{X} \in \C^{N \times N}$ be a mean-zero image.
Then
\begin{equation}
\label{SobStand}
\| \mtx{X} \|_2 \leq \| \mtx{X} \|_{TV}
\end{equation}
\end{proposition}
This inequality also holds if instead of being mean-zero, $\mtx{X} \in \C^{N \times N}$ contains some zero-valued pixel.  In the appendix, we give a direct proof of the Sobolev inequality \eqref{SobStand} in the case that all pixels in the first column and first row of $\mtx{X} \in \C^{N \times N}$ are zero-valued, $X_{1,j} = X_{j,1} = 0.$

The Sobolev inequality can be used to derive image error guarantees given gradient error guarantees.  However, we will be able to derive even sharper estimates by appealing to a remarkable theorem from \cite{cdpx} which says that the bivariate Haar coefficient vector of a zero-mean function $f \in BV(Q)$ on the unit square $Q = [0,1)^2$ is in weak $\ell_1$, and its weak $\ell_1$-norm is proportional to its bounded-variation semi-norm.  The following proposition is a corollary of Theorem $8.1$ of \cite{cdpx}, and the derivation is given in the appendix.   
 \begin{proposition}
\label{cor:cdpx}
Suppose $\mtx{X} \in \C^{N \times N}$ is mean-zero, and let $c_{(k)}(\mtx{X})$ be the entry of the bivariate Haar transform ${\cal H}(\mtx{X})$ having $k$th largest magnitude. Then for all $k \geq 1$, 
$$
| c_{(k)}(\mtx{X}) | \leq C_1 \frac{ \| \mtx{X} \|_{TV}}{k}
$$
where $C_1 = 36(480 \sqrt{5} + 168 \sqrt{3})$. 
\end{proposition}

Proposition \ref{cor:cdpx} bounds the decay of the Haar wavelet coefficients by the image total variation semi-norm.  At the same time, vectors lying in the null space of a matrix with the restricted isometry property must be sufficiently \emph{flat}, with the $\ell_2$-energy in their largest $s$ components in magnitude bounded by the $\ell_1$ norm of the remaining components (the so-called \emph{null-space property}) \cite{codade09}.  As a result, the $\ell_2$ norm of the bivariate Haar transform  of  $\mtx{D} = \mtx{X}-\mtx{\hat{X}}$, and thus the $\ell_2$ norm of $\mtx{D}$ itself, must be sufficiently small.  In fact, $\mtx{D}$ satisfies a Sobolev inequality that is stronger than the standard inequality \eqref{SobStand} by a factor of $\log(N^2/s)/\sqrt{s}$. 

\begin{theorem}[Strong Sobolev inequality]
\label{strongsobo}
Consider $s, m, n \in \N$, and let $N = 2^n$.  Let ${\cal B}: \C^{N \times N} \rightarrow \C^m$ be a linear map which, composed with the inverse bivariate Haar transform ${\cal B} {\cal H}^{-1}: \C^{N \times N} \rightarrow \C^{m}$, has the restricted isometry property of order $2s$ and level $\delta < 1$.  Suppose $\mtx{X} \in \C^{N \times N}$ satisfies $\| {\cal B}(\mtx{X}) \|_2 \leq \varepsilon$. Then
\begin{equation}
\label{sobstrong}
\| \mtx{X} \|_2 \leq C' s^{-1/2} \log(N^2/s) \| \mtx{X} \|_{TV}  +C'' \varepsilon
\end{equation}
where $C' = \frac{2C_1}{(1-\delta)}$ and $C_1$ the constant from Proposition \ref{cor:cdpx}, and $C'' = \frac{1}{1-\delta}$.
In particular, if ${\bf X} \in \C^{N \times N}$ lies in the null space of ${\cal B}$, then
$$
\| \mtx{X} \|_2 \lesssim s^{-1/2} \log(N^2/s) \| \mtx{X} \|_{TV}  
$$
\end{theorem}

\begin{proof}
Let $\vct{c} = {\cal H}({\mtx X}) \in \C^{N^2}$ be the bivariate Haar wavelet transform of $\mtx{X}$, regarded as a vector in $\mtx{C}^{N^2}$.  First decompose $\vct{c} = \vct{c}_0 + \vct{c}_1$, where $\vct{c}_0$ is one-sparse and consists only of the first Haar coefficient $c_0 = \scalprod{ \mtx{X}}{\vct{h}^0 } = \frac{1}{N} \sum_{j,k} X_{j,k}$.   Write $\mtx{X} = \mtx{X}_0 + \mtx{X}_1$, where $\mtx{X}_0 \equiv \frac{c_0}{N}$ is constant and $\mtx{X}_1$ is mean-zero.   Note that $(0, \vct{c}_1) = {\cal H}(\mtx{X}_1)$ and $(\vct{c}_0, {\bf 0}) = {\cal H}(\mtx{X}_0)$.
Denote the $j$th-largest magnitude entry of $\vct{c}_1$ by $c_{(j)}$, and write $\vct{c}_1 = \vct{c}_{\Omega} + \vct{c}_{\Omega^c}$ where $\vct{c}_{\Omega}$ is the $s$-sparse vector of $s$ largest-magnitude entries of $\vct{c}_1$.  Further decompose $\vct{c}_{\Omega^c} = \vct{c}_{\Omega^c}^{1} + \vct{c}_{\Omega^c}^{2} + \dots + \vct{c}_{\Omega^c}^{r}$ where $r = \lfloor \frac{N^2}{s} \rfloor$, and $\vct{c}_{\Omega^c}^{1}$ is the $s$-sparse image pointing to the $s$ largest-magnitude entries remaining in $\vct{c}_{\Omega^c}$, and $\vct{c}_{\Omega^c}^{2}$ is the $s$-sparse image pointing to the $s$ largest-magnitude entries remaining in $\vct{c}_{\Omega^c} - \vct{c}_{\Omega^c}^{1}$, and so on.  

By Proposition \ref{cor:cdpx} we know that $| c_{(j)} | \leq C_1\| \mtx{X} \|_{TV} / j$.  Then
\begin{eqnarray}
\label{tail1}
\| \vct{c}_{\Omega^c} \|_1 &=& \sum_{j = s+1}^{N^2} |c_{(j)}| \nonumber \\
&\leq& C_1\| \mtx{X} \|_{TV} \sum_{j = s+1}^{N^2} \frac{1}{j} \nonumber \\
&\leq&C_1\| \mtx{X} \|_{TV} \log(N^2/s).
\end{eqnarray}
We can similarly bound the $\ell_2$-norm:
\begin{eqnarray}
\label{tail2}
\| \vct{c}_{\Omega^c} \|_2^2 &=& \sum_{j=s+1}^{N^2} |c_{(j)} |^2 \nonumber \\
 &\leq& C_1^2 \| \mtx{X} \|_{TV}^2 \sum_{j = s+1}^{N^2} \frac{1}{j^2} \nonumber \\
 &\leq& C_1^2 \| \mtx{X} \|_{TV}^2/s,
\end{eqnarray}
obtaining $\| \vct{c}_{\Omega^c} \|_2 \leq C_1 s^{-1/2} \| \mtx{X} \|_{TV}$.

We now use the restricted isometry property for ${\cal B} {\cal H}^{-1}$ and assumed tube constraint $\| {\cal B}(\mtx{X}) \|_2 \leq \varepsilon$ to obtain
\begin{eqnarray}
\varepsilon &\geq& \| {\cal B}(\mtx{X}) \|_2 \nonumber \\
&=& \| {\cal B}{\cal H}^{-1} (\vct{c}_0 + \vct{c}_{\Omega} + \vct{c}_{\Omega^c}) \|_2 \nonumber \\
&\geq& \| {\cal B} {\cal H}^{-1} (\vct{c}_0 + \vct{c}_{\Omega} + \vct{c}^1_{{\Omega}^c}) \|_2 - \sum_{j=2}^r  \| {\cal B}{\cal H}^{-1} (\vct{c}_{\Omega^c}^{j}) \|_2 \nonumber \\
&\geq& (1 - \delta) \| \vct{c}_0 + \vct{c}_{\Omega} + \vct{c}_{\Omega^c}^{1} \|_2 -(1 + \delta)  \sum_{j=2}^r \| \vct{c}_{\Omega^c}^{j} \|_2 \nonumber \\
&\geq& (1 - \delta) \| \vct{c}_0 + \vct{c}_{\Omega} \|_2 -(1 + \delta)  \sum_{j=2}^r \| \vct{c}_{\Omega^c}^{j} \|_2 \nonumber \\
&\geq& (1- \delta) \| \vct{c}_0 + \vct{c}_{\Omega} \|_2 - (1 + \delta) s^{-1/2} \sum_{j=1}^r \| \vct{c}_{\Omega^c}^{j} \|_1 \nonumber \\
&=&  (1- \delta) \| \vct{c}_0 + \vct{c}_{\Omega} \|_2 - (1 + \delta) s^{-1/2} \| \vct{c}_{\Omega^c} \|_1.
\end{eqnarray}
In the final inequality we applied the block-wise bound $\| \vct{c}_{\Omega^c}^{j} \|_2 \leq s^{-1/2} \| \vct{c}_{\Omega^c}^{j-1} \|_1$, which holds because the magnitude of each entry of $\vct{c}_{\Omega^c}^{j-1}$ is larger than the average magnitude of the entries $\vct{c}_{\Omega^c}^{j}$.

Combined with the bound \eqref{tail1} on $\| \vct{c}_{\Omega^c} \|_1$ this gives 
\begin{eqnarray}
\|  \vct{c}_0 + \vct{c}_{\Omega} \|_2 &\leq& \frac{\varepsilon}{1-\delta} + \frac{C_1(1+\delta)}{s^{1/2}(1-\delta)} \log(N^2/s)\| \mtx{X} \|_{TV}
\end{eqnarray}
Together with the bound \eqref{tail2} on $\| \vct{c}_{\Omega^c} \|_2 $ and orthonormality of the bivariate Haar transform, we find that
\begin{eqnarray}
\| \mtx{X} \|_2 =  \| \vct{c} \|_2  \leq \| \vct{c}_0 + \vct{c}_{\Omega} \|_2 + \| \vct{c}_{\Omega^c} \|_2  &\leq \frac{\varepsilon}{1-\delta} + \frac{2C_1 \log(N^2/s)}{s^{1/2}(1-\delta)}\| \mtx{X} \|_{TV}  \nonumber 
\end{eqnarray}
This completes the proof.
\end{proof}

\subsection{Proof of Theorem \ref{thm:gen}}\label{sec:prfgen}
Since bounds~\eqref{stable1} and~\eqref{stable2} were already proven in Section~\ref{stabgrad}, it remains to prove the stability bound~\eqref{stable3}.
Given measurements of $\mtx{X}$ of the form~\eqref{measure:M}, the image error $\mtx{D} = \mtx{X} - \mtx{\hat{X}}$ satisfies the tube-constraint $\| {\cal B}(\mtx{D}) \|_2 \leq \varepsilon$.  Thus the bound \eqref{stable2} on $\| \mtx{D} \|_{TV}$ along with Theorem \ref{strongsobo} give
\begin{align*}
\| \mtx{X}-\mtx{\hat{X}} \|_2 = \| \mtx{D} \|_2 &\lesssim \varepsilon + \log(N^2/s)\Big( \frac{\| \mtx{D} \|_{TV}}{\sqrt{s}}\Big)\\
 &\lesssim \varepsilon + \log(N^2/s)\Big( \frac{\| \mtx{\nabla}\mtx{X} - (\mtx{\nabla}\mtx{X})_s \|_1 + \sqrt{s} \varepsilon}{\sqrt{s}}\Big)\\
 &\lesssim  \log(N^2/s)\Big( \frac{ \| \mtx{\nabla}\mtx{X} - (\mtx{\nabla}\mtx{X})_s \|_1}{\sqrt{s}} + \varepsilon\Big) .
\end{align*}
 This completes the proof of Theorem~\ref{thm:gen}. 

\section{Proof of Theorem \ref{thm2}}\label{sec:thm2}

We will need two preliminary lemmas about the bivariate Haar system.
\begin{lemma}
\label{edge}
Let $N = 2^n$, and consider the discrete bivariate Haar wavelet basis functions $\vct{h}^0$ and $(\vct{h}^e_{p,\ell}) \in \C^{N \times N}$.  For any fixed pair of indices $(j, k)$ and $(j, k +1)$ or $(j,k)$ and $(j+1,k)$, there are at most $6n$ bivariate Haar wavelets which are not constant on these indices.
\end{lemma}
\begin{proof}
The lemma follows by showing that for fixed dyadic scale $p$ between 1 and $n$, there are at most 6 Haar wavelets with side dimension $2^{n-p}$ which are not constant on these two indices.  Indeed, if the edge between $(j, k)$ and $(j, k+1)$ coincides with a dyadic edge at scale $p$, then the 3 wavelets supported on each of the two adjacent dyadic squares transition from being zero to nonzero along this edge.  The only other case to consider is that $(j, k)$ coincides with a dyadic edge at dyadic scale $p+1$ but does not coincide with a dyadic edge at scale $p$; in this case the 3 wavelets supported on the dyadic square centered at $(j, k+1), (j, k)$ can change from negative to positive value.      
\end{proof}

\begin{lemma}
\label{l1haar}
The bivariate Haar wavelets have uniformly bounded total variation: $\| \nabla (\vct{h}_{p,\ell}^e) \|_1 \leq 8$.
\end{lemma}
\begin{proof}
The wavelet $\vct{h}_{p,\ell}^e$ is supported on a dyadic square of side-length $2^{n-p}$, and on its support has constant magnitude $| (\vct{h}_{p,\ell}^e)_{j,k} | \equiv 2^{p-n}$ and changes sign along the the horizontal and vertical lines bisecting the square on which it is supported.  A little bit of algebra gives that $\| \nabla (\vct{h}_{p,\ell}^e) \|_1  \leq 8 \cdot 2^{n-p} \cdot 2^{p-n} = 8$.
\end{proof}

\noindent We are now in a position to prove Theorem \ref{thm2}.  

\vspace{3mm}

\noindent{\bf Proof of Theorem \ref{thm2}.}  
Let $\mtx{D} = \mtx{X} - \mtx{\hat{X}}$ denote the residual error by (TV). 
Let $\vct{c} = {\cal H}({\mtx D}) \in \C^{N^2}$ be the bivariate Haar wavelet transform of $\mtx{D}$, regarded as a vector.  First decompose $\vct{c} = \vct{c}_0 + \vct{c}_1$, where $\vct{c}_0 = \scalprod{ \mtx{D}}{\vct{h}^0 } = \frac{1}{N} \sum_{j,k} D_{j,k}$ is the coefficient of the constant Haar wavelet.  Write $\mtx{D} = \mtx{D}_0 + \mtx{D}_1$, where $\mtx{D}_0 \equiv \frac{\vct{c}_0}{N}$ is constant and $\mtx{D}_1$ is mean-zero.   Finally, let $c_{(j)}$ denote the $j$th largest-magnitude Haar coefficient of $\mtx{D}_1$, and let $\vct{h}_{(j)}$ be the Haar wavelet associated to $c_{(j)}$.

By assumption, ${\cal A}{\cal H}^{-1}$ has the restricted isometry property of level $\delta < 1/3$ and order 
\begin{equation}
\label{arip}
\widetilde{s} = 4320 C_1^2 s \log^3(N),
\end{equation}
where $C_1$ is the constant from Proposition \ref{cor:cdpx}.

\begin{description}
\item[Cone constraint on $\nabla \mtx{D}$.] Let $S \subset N \times N \times 2$ be the subset of $s$ largest-magnitude entries of $\nabla \mtx{D}$.  Note that $S$ can be identified with $s$ index pairs $(j,k), (j+1,k)$ or $(j,k), (j,k+1) \in \C^{N \times N}.$  
As shown in Section \ref{stabgrad}, we have the cone constraint 
\begin{equation}
\label{coned}
\| (\mtx{\nabla}\mtx{D})_{S^c} \|_1 \leq \| (\mtx{\nabla}\mtx{D})_S  \|_1 + 2\| \mtx{\nabla}\mtx{X} - (\mtx{\nabla}\mtx{X})_s \|_1.
\end{equation}
\item[Cone constraint on $\vct{c} = {\cal H}(\mtx{D})$.]  Proposition \ref{cor:cdpx} allows us to pass from a cone constraint on $\nabla \mtx{D}$ to a cone constraint on ${\cal H}(\mtx{D})$.  By Lemma \ref{edge},  there are at most $6s n = 6s \log(N)$ wavelets which are non-constant over the edges indexed by $S$.  Let $\Omega \subset N \times N$ refer to this set of wavelets.  Decompose $\mtx{D}$ as
\begin{equation}
\label{d_decomp}
\mtx{D} = \sum_{j} c_{(j)} \vct{h}_{(j)} =  \sum_{j \in \Omega} c_{(j)} \vct{h}_{(j)}  + \sum_{j \in \Omega^c} c_{(j)} \vct{h}_{(j)}  =: \mtx{D}_1 + \mtx{D}_2. 
\end{equation}
By linearity of the gradient,  $\nabla \mtx{D} = \nabla \mtx{D}_1 + \nabla \mtx{D}_2.$ 
By construction of $\Omega$, we have immediately that $( \nabla \mtx{D}_2)_S = 0$, leaving us with $  (\nabla \mtx{D})_S  = (\nabla \mtx{D}_1 )_S$.
By Lemma \ref{l1haar} and the triangle inequality, 
\begin{eqnarray}
\|( \nabla \mtx{D} )_S \|_1 = \| (\nabla \mtx{D}_1 )_S \|_1 &\leq& \|  \nabla {\mtx{D}_1} \|_1 \nonumber \\
&=& \| \nabla\big( \sum_{j \in \Omega} c_{(j)} \vct{h}_{(j)} \big) \|_1 \nonumber \\
&\leq& \sum_{j \in \Omega} | c_{(j)} |  \|\nabla \vct{h}_{(j)} \|_1 \nonumber \\
&\leq& 8 \sum_{j \in \Omega} | c_{(j)} |.
\end{eqnarray}
Combined with Proposition \ref{cor:cdpx} and the cone constraint \eqref{coned}, and letting 
$$k = 6s \log(N) = | \Omega |,$$
we deduce the string of inequalities
\begin{align*}
\sum_{j = k + 1}^{N^2-1} | c_{(j)} | &\leq \sum_{j = s+1}^{N^2} | c_{(j)} | \nonumber \\
 &\leq C_1 \log(N^2/s)  \| \nabla \mtx{D} \|_{1}\nonumber \\
&= C_1 \log(N^2/s) \Big( \| (\nabla \mtx{D})_S \|_1 +  \| (\nabla \mtx{D})_{S^c} \|_1  \Big)\nonumber \\ \nonumber \\
&\leq C_1 \log(N^2/s)\Big( 2\| (\nabla \mtx{D})_S \|_1 + 2\| \nabla \mtx{X} - (\nabla \mtx{X})_S \|_1\Big) \nonumber \\ \nonumber \\
&\leq C_1 \log(N^2/s)\Big(12\sum_{j \in \Omega} | c_{(j)} | + 2\| \nabla \mtx{X} - (\nabla \mtx{X})_S \|_1\Big)   \nonumber \\ \nonumber \\
&\leq 12C_1  \log(N^2/s)\Big( \sum_{j =1}^{k } | c_{(j)} | + \| \nabla \mtx{X} - (\nabla \mtx{X})_S \|_1\Big)   \nonumber \\
&\leq 12C_1  \log(N^2/s)\Big( c_0 + \sum_{j =1}^{k } | c_{(j)} | + \| \nabla \mtx{X} - (\nabla \mtx{X})_S \|_1\Big)   
\end{align*}

\item[Tube constraint  $\| {\cal A} {\cal H}^{-1}(\vct{c}) \|_2 \leq 2 \varepsilon.$]
By assumption, ${\cal A}{\cal H}^{-1}: \C^{N^2} \rightarrow \C^m$ has the RIP of order $\widetilde s > k $. 
Since both $\mtx{X}$ and $\mtx{\hat{X}}$ are in the feasible region of (TV), we have for $\vct{c} = {\cal H}(\mtx{D}) = {\cal H}(\mtx{X})- {\cal H}(\mtx{\hat{X}})$,
\begin{align*}
\| {\cal A} {\cal H}^{-1}(\vct{c}) \|_2 \leq  \| {\cal A} (\mtx{X})\|_2 + \| {\cal A}(\mtx{\hat{X}}) \|_2 \leq  2\varepsilon.
\end{align*}
\end{description}

 Using the derived cone and tube constraints on $\vct{c} = {\cal H}(\mtx{D})$, we apply Proposition~\ref{cone-tube} using $\gamma = 12 C_1 \log(N^2/s)$, $k = 6s\log N$, and $\sigma = 12 C_1 \log(N^2/s) \| \nabla \mtx{X} - (\nabla \mtx{X} )_S \|_1$ to complete the proof.  In fact, this  is where we need that the RIP order is $\widetilde s$ in \eqref{arip}, to accommodate for the factors $\gamma$ and $k$ in Proposition~\ref{cone-tube}.

\section{Conclusion}\label{sec:conc}  

Compressed sensing techniques provide reconstruction of compressible signals from few linear measurements.  A fundamental application is image compression and reconstruction.  Since images are compressible with respect to wavelet bases, standard CS methods such as $\ell_1$-minimization guarantee reconstruction to within a factor of the error of best $s$-term wavelet approximation.  The story does not end here, though.  Images are more compressible with respect to their discrete gradient representation, and indeed the advantages of total variation (TV) minimization over wavelet-coefficient minimization have been empirically well documented (see e.g.~\cite{candes05pr,CRT06:Stable}). It had been well-known that without measurement noise, images with perfectly sparse gradients are recovered exactly via TV-minimization~\cite{crt}.  Of course in practice, images do not have exactly sparse gradients, and measurements are corrupted with additive or quantization noise.  To our best knowledge, our main results, Theorems~\ref{thm:gen} and \ref{thm2}, are the first to provably guarantee robust image recovery via TV-minimization.  In analog to the standard compressed sensing results, the number of measurements in Theorem~\ref{thm:gen} required for reconstruction is optimal, up to a single logarithmic factor in the image dimension.  Theorem~\ref{thm:gen} has been extended to the multidimensional case, for signals with higher dimensional structure such as movies~\cite{NW12:TV3}.  On the other hand, the proof of Theorem~\ref{thm2} is specific to properties of the bivariate Haar system, and extending it to higher dimensions (as well as for $d=1$) remains an open problem.  Theorem~\ref{thm2} applies, for example, to partial Fourier matrices subsampled according to appropriate variable densities~\cite{kw12}.  Finally, we believe our proof technique can be used for analysis operators beyond the total variation operator.  For example, in practice one often finds that minimizing a sum of TV and wavelet norms yields improved image recovery.   We leave this and the study of more general analysis type operators as future work.

\subsection*{Acknowledgment}
We would like to thank Arie Israel, Christina Frederick, Felix Krahmer, Stan Osher, Yaniv Plan, Justin Romberg, Joel Tropp, and Mark Tygert for invaluable discussions and improvements.  We also would like to thank Fabio Lanzoni and his agent, Eric Ashenberg.  Rachel Ward acknowledges the support of a Donald D. Harrington Faculty Fellowship, Alfred P. Sloan Research Fellowship, and DOD-Navy grant N00014-12-1-0743.

\appendix

\section{Proofs of Lemmas and Propositions}

\subsection{Proof of Proposition \ref{cone-tube}}

Here we include a proof of Proposition~\ref{cone-tube}, which is a modest generalization of results from~\cite{CRT06:Stable}.

Let $s = k \gamma^2$ and let $S \subset [N]$ be the support set of the best $s$-term approximation of $\mtx{D}$.
\begin{proof}
By assumption, we suppose that $\mtx{D}$ obeys the cone constraint
\begin{equation}
\label{cone}
\| \mtx{D}_{S^c}\|_1 \leq \gamma \| \mtx{D}_S \|_1 + \sigma
\end{equation}
and the tube constraint $\| {\cal A}( \mtx{D}) \|_2 \leq \varepsilon$.

We write $\mtx{D}_{S^c} = \mtx{D}_{S_1} + \mtx{D}_{S_2}+ \dots + \mtx{D}_{S_r}$ where $r = \lfloor{ \frac{N^2}{4s} \rfloor}$.  Here $\mtx{D}_{S_1}$ consists of the $4s$ largest-magnitude components of $\mtx{D}$ over $S^c$, $\mtx{D}_{S_2}$ consists of the $4s$ largest-magnitude components of $\mtx{D}$ over $S^c \setminus S_1$, and so on.  Note that $\mtx{D}_{S}$ and similar expressions below can have both the meaning of restricting $\mtx{D}$ to the indices in $S$ as well as being the array whose entries are set to zero outside $S$.

Since the magnitude of each nonzero component of $\mtx{D}_{S_{j-1}}$ is larger than the average magnitude of the nonzero components of $\mtx{D}_{S_j}$,
$$
\| \mtx{D}_{S_j} \|_2 \leq \frac{\| \mtx{D}_{S_{j-1}} \|_1}{2\sqrt{s}}, \quad j=2,3,\dots
$$
Combining this with the cone constraint gives
\begin{equation}
\sum_{j=2}^r \| \mtx{D}_{S_j} \|_2 \leq \frac{1}{2\gamma \sqrt{k}} \| \mtx{D}_{S^c} \|_1 \leq \frac{1}{2\sqrt{k}} \| \mtx{D}_S \|_1 + \frac{1}{2 \gamma \sqrt{k}} \sigma \leq \frac{1}{2} \| \mtx{D}_S \|_2 + \frac{1}{2\gamma \sqrt{k}}\sigma
\end{equation}

Now combined with the tube constraint and the RIP,

\begin{eqnarray}
\varepsilon &\gtrsim& \| {\cal A} \mtx{D} \|_2 \nonumber \\
&\geq& \| {\cal A} (\mtx{D}_S + \mtx{D}_{S_1}) \|_2 - \sum_{j=2}^r \| {\cal A}( \mtx{D}_{S_j}) \|_2 \nonumber \\
&\geq& \sqrt{1 - \delta} \| \mtx{D}_S + \mtx{D}_{S_1} \|_2 - \sqrt{1 + \delta} \sum_{j=2}^r \| \mtx{D}_{S_j} \|_2 \nonumber \\
&\geq& \sqrt{1 - \delta} \| \mtx{D}_S + \mtx{D}_{S_1} \|_2 - \sqrt{1 + \delta} \Big( \frac{1}{2} \| \mtx{D}_S \|_2 + \frac{1}{2\gamma \sqrt{k}}\sigma\Big) \nonumber \\
&\geq& \Big( \sqrt{1 - \delta} - \frac{\sqrt{1+\delta}}{2} \Big)  \| \mtx{D}_S + \mtx{D}_{S_1} \|_2 - \sqrt{1 + \delta} \frac{1}{2\gamma \sqrt{k}}\sigma
\end{eqnarray}
Then since $\delta < 1/3$,
$$
\| \mtx{D}_S + \mtx{D}_{S_1} \|_2 \leq 5\varepsilon + \frac{3\sigma}{\gamma \sqrt{k}}.
$$
Finally, because $\| \sum_{j=2}^r \mtx{D}_{S_j} \|_2  \leq \sum_{j=2}^r \| \mtx{D}_{S_j} \|_2 \leq \frac{1}{2} \| \mtx{D}_S + \mtx{D}_{S_1} \|_2 + \frac{1}{2\gamma \sqrt{k}}\sigma$, we have
$$
\| \mtx{D} \|_2 \leq 8\varepsilon + \frac{5\sigma}{\gamma \sqrt{k}},
$$
confirming \eqref{eq:h2}.

To confirm \eqref{eq:h1}, note that the cone constraint allows the estimate
\begin{eqnarray}
\| \mtx{D} \|_1 &\leq& (\gamma + 1) \| \mtx{D}_S \|_1 + \sigma \nonumber \\
&\leq& 2\gamma \sqrt{s} \| \mtx{D}_S \|_2 + \sigma \nonumber \\
&\leq& 2\gamma \sqrt{k} \left(5\varepsilon + \frac{3\sigma}{\gamma\sqrt{k}}\right) + \sigma
\end{eqnarray}

\end{proof}

\subsection{Proof of Proposition~\ref{corav}}

Here we give a direct proof of the discrete Sobolev inequality \eqref{SobStand} for images $\mtx{X} \in \C^{N \times N}$ whose first row and first column of pixels are zero-valued,  $X_{1,j} = X_{j,1} = 0.$ 

\begin{proof}
For any $1 \leq k \leq i \leq N$ we have,
\begin{eqnarray}
\label{coraveq}
|X_{i,j}| &=& \Big|X_{1,j} + \sum_{\ell=1}^{i-1} \big( X_{\ell+1,j} - X_{\ell,j} \big)\Big| \nonumber \\
&\leq& \sum_{\ell=1}^{i-1} | X_{\ell+1,j} - X_{\ell,j} | \nonumber \\
&\leq& \sum_{\ell = 1}^{N-1} | X_{\ell+1,j} - X_{\ell,j} |.
\end{eqnarray}
Similarly, by reversing the order of indices we also have
\begin{equation}
\label{corav2}
|X_{i,j}| \leq  \sum_{\ell=1}^{N-1} | X_{i,\ell+1} - X_{i,\ell} |.
\end{equation}
For ease of notation let
$$f(j) =  \sum_{\ell=1}^{N-1} | X_{\ell+1,j} - X_{\ell,j} |$$ 
and let
$$g(i) = \sum_{\ell=1}^{N-1} | X_{i,\ell+1} - X_{i,\ell} |.$$
Combining the two bounds \eqref{coraveq} and \eqref{corav2} on $X_{i,j}$ results in the bound
$| X_{i,j}|^2 \leq f(j) \cdot g(i).$ 

Summing this inequality over all pixels $(i,j)$,
\begin{eqnarray}
\| \mtx{X} \|^2 = \sum_{i=1}^N \sum_{j=1}^N | X_{i,j}|^2 &\leq& \Big( \sum_{j=1}^N f(j) \Big) \Big( \sum_{i=1}^N g(i) \Big) \nonumber \\
&\leq& \frac{1}{4} \cdot \left(  \sum_{j=1}^N f(j) + \sum_{i=1}^N g(i) \right)^2 \nonumber \\
&\leq& \frac{1}{4} \cdot \left(  \sum_{j=1}^N \sum_{k=1}^{N-1} | X_{k+1,j} - X_{k,j} | + \sum_{i=1}^N  \sum_{k=1}^{N-1} | X_{i,k+1} - X_{i,k} | \right)^2 \nonumber \\
&\leq& \frac{1}{4} \| \mtx{\nabla}\mtx{X} \|^2_1 \nonumber \\
&=&  \frac{1}{4} \| \mtx{X} \|^2_{TV}.
\end{eqnarray} 
\end{proof}

\subsection{Derivation of Proposition \ref{cor:cdpx}} 

Recall that a function $f(u,v)$ is in the space $L_p(\Omega)$ $(1 \leq p < \infty)$ if
$$
\| f \|_{L_p(\Omega)} := \Big( \int_{\Omega \subset \R^2} | f(x) |^p dx \Big)^{1/p} < \infty,
$$
and the space of functions with bounded variation on the unit square is defined as follows.
\begin{definition}
$BV(Q)$ is the space of functions of bounded variation on the unit square $Q := [0,1)^2 \subset \mathbb{R}^2$.
For a vector $\vct{v} \in \mathbb{R}^2$, we define the difference operator $\Delta_{\vct{v}}$ in the direction of $\vct{v}$ by 
$$
\Delta_{\vct{v}}(f,\vct{x}) := f(\vct{x} + \vct{v}) - f(\vct{x}).
$$ 
We say that a function $f \in L_1(Q)$ is in $BV(Q)$ if and only if
$$
V_Q(f) := \sup_{h > 0} h^{-1} \sum_{j=1}^2 \| \Delta_{h \vct{e}_j}(f,\cdot) \|_{L_1(Q(h\vct{e}_j))} =\lim_{h \rightarrow 0}h^{-1} \sum_{j=1}^2  \| \Delta_{h \vct{e}_j} (f, \cdot) \|_{L_1(Q(h \vct{e}_j))}
$$
is finite, where $\vct{e}_j$ denotes the $j$th coordinate vector.  Here, the last equality follows from the fact that $\| \Delta_{h \vct{e}_j} (f,\cdot) \|_{L_1(Q)}$ is subadditive.  
$V_Q(f)$ provides a semi-norm for BV:
$$
| f |_{BV(Q)} := V_Q(f).
$$
\end{definition}

Theorem $8.1$ of \cite{cdpx} bounds the rate of decay of a function's bivariate Haar coefficients by its bounded variation semi-norm.

\begin{theorem}[Theorem $8.1$ of \cite{cdpx}]
\label{haardecaybv}
Consider a mean-zero function $f \in BV(Q)$ and its bivariate Haar coefficients arranged in decreasing order according to their absolute value, $c_{(k)}(f)$.  We have
$$
c_{(k)}(f) \leq C_1 \frac{ |f |_{BV} }{k}
$$
where $C_1 = 36(480\sqrt{5} + 168\sqrt{3})$.
\end{theorem}
As discrete images are isometric to piecewise-constant functions of the form \eqref{isometry}, the bivariate Haar coefficients of the image $\mtx{X} \in \C^{N \times N}$ are equal to those of the function $f_{\mtx{X}} \in L_2(Q)$ given by 
\begin{equation}
\label{xtoX}
f_{\mtx{X}}(u, v) =N \mtx{X}_{i,j}, \quad \frac{i-1}{N} \leq u < \frac{i}{N}, \quad  \frac{j-1}{N} \leq v < \frac{j}{N}, \quad 1 \leq i,j \leq N.
\end{equation} 
To derive Proposition \ref{cor:cdpx}, it will suffice to verify that the \emph{bounded variation} of $f_{\mtx{X}}$ can be bounded by the \emph{total variation} of $\mtx{X}$.

\begin{lemma}
$| f_{\mtx{X}} |_{BV} \leq \| \mtx{X} \|_{TV}$
\end{lemma}

\begin{proof}
For ${h} < \frac{1}{N}$, 
\begin{equation}
\label{delta1}
\Delta_{h e_1} \big( f_{\mtx{X}}, (u,v) \big) = \left\{ \begin{array}{ll} 
N (\mtx{X}_{i+1,j} - \mtx{X}_{i,j} ) & \frac{i}{N} - h \leq u \leq \frac{i}{N}, \quad \frac{j}{N} \leq v \leq \frac{j+1}{N}, \nonumber \\
0, & \textrm{else}. 
\end{array} \right .
\end{equation}
and
\begin{equation}
\label{delta2}
\Delta_{h e_2} \big( f_{\mtx{X}},(u,v) \big) = \left\{ \begin{array}{ll} 
N( \mtx{X}_{i,j+1} - \mtx{X}_{i,j}) ,  & \frac{i}{N}  \leq u \leq \frac{i+1}{N}, \quad \frac{j}{N}-h \leq v \leq \frac{j}{N}, \nonumber \\
0, & \textrm{else}. 
\end{array} \right .
\end{equation}
Then
\begin{eqnarray}
| f_{\mtx{X}} |_{BV} &=& \lim_{h \rightarrow 0} \frac{1}{h} \left[  \int_{0}^1 \int_{0}^1 | f_{\mtx{X}}(u+h, v) - f_{\mtx{X}}(u,v) | \hspace{1mm} du dv + \int_{0}^1 \int_{0}^1 |f_{\mtx{X}} (u, v+h) - f_{\mtx{X}}(u,v) | \hspace{1mm} dv du  \right]  \nonumber\\
&=& \sum_{j=1}^{N-1} \sum_{i=1}^{N-1}  | \mtx{X}_{i+1,j}-\mtx{X}_{i,j} | + \sum_{i=1}^{N-1} \sum_{j=1}^{N-1} | \mtx{X}_{i,j+1} - \mtx{X}_{i,j} |  \nonumber\\
&\leq& \| \mtx{X} \|_{TV}
\end{eqnarray}
\end{proof}

\bibliographystyle{plain}
\bibliography{haar}

\end{document}